\documentclass[letterpaper, 10 pt, conference]{ieeeconf}  

\IEEEoverridecommandlockouts                             

\usepackage{cite}
\usepackage{placeins}

\usepackage{amsmath,amssymb,hyperref,amsthm}

\newtheorem{theorem}{Theorem}
\newtheorem{lemma}{Lemma}
\newtheorem{definition}{Definition}

\renewenvironment{proof}[1][Proof]{
  \par
  \noindent
  \textbf{#1.}
  \quad
  \ignorespaces
}{
  \unskip
  \quad
  \qed
  \par
  \vspace{6pt}
}

\usepackage{algorithm} 
\usepackage{algpseudocode} 

\usepackage{graphicx}
\usepackage{subfig}

\usepackage{xcolor}

\newcommand{\blue}[1]{\textcolor{black}{#1}}

\title{\LARGE \bf
A Generalized Voronoi Graph based Coverage Control Approach for Non-Convex Environment
}

\author{Albert Author$^{1}$ and Bernard D. Researcher$^{2}$
\thanks{*This work was not supported by any organization}
\thanks{$^{1}$Albert Author is with Faculty of Electrical Engineering, Mathematics and Computer Science,
        University of Twente, 7500 AE Enschede, The Netherlands
        {\tt\small albert.author@papercept.net}}%
\thanks{$^{2}$Bernard D. Researcheris with the Department of Electrical Engineering, Wright State University,
        Dayton, OH 45435, USA
        {\tt\small b.d.researcher@ieee.org}}%
}

\author{Zuyi Guo, Ronghao Zheng$^\dag$, Meiqin Liu, Senlin Zhang%
\thanks{This work was supported by the Zhejiang Provincial Natural Science Foundation of China under Grant LZ24F030001, and the National Natural Science Foundation of China under Grants U23B2060 and 62173294. 
(\textit{Corresponding author: Ronghao Zheng.})}
\thanks{Zuyi Guo and Ronghao Zheng, and Senlin Zhang are with the College of Electrical Engineering, Zhejiang University, Hangzhou 310027, China, and also with the State Key Laboratory of Industrial Control Technology, Zhejiang University, Hangzhou 310027, China. (e-mail: {\{3220105160, rzheng, slzhang\}@zju.edu.cn})
}
\thanks{Meiqin Liu is with the National Key Laboratory of Human-Machine Hybrid Augmented Intelligence, Xi’an Jiaotong University, Xi’an 710049, China, and also with the State Key Laboratory of Industrial Control Technology, Zhejiang University, Hangzhou 310027, China. (e-mail: {liumeiqin@zju.edu.cn})}
}

\begin{document}

\maketitle
\thispagestyle{empty}
\pagestyle{empty}

\begin{abstract}

To address the challenge of efficient coverage by multi-robot systems in non-convex regions with multiple obstacles, this paper proposes a coverage control method based on the Generalized Voronoi Graph (GVG), which has two phases: Load-Balancing Algorithm phase and Collaborative Coverage phase. In Load-Balancing Algorithm phase, the non-convex region is partitioned into multiple sub-regions based on GVG. Besides, a weighted load-balancing algorithm is developed, which considers the quality differences among sub-regions. By iteratively optimizing the robot allocation ratio, the number of robots in each sub-region is matched with the sub-region quality to achieve load balance. In Collaborative Coverage phase, each robot is controlled by a new controller to effectively coverage the region. The convergence of the method is proved and its performance is evaluated through simulations.

\end{abstract}

\section{INTRODUCTION}

Recent advancements in multi-robot systems have revolutionized distributed problem-solving paradigms by enabling autonomous agents to collaboratively achieve complex objectives through local interactions. Coverage control, as one of the key branches, has attracted a lot of attentions in the past decades, with widespread applications ranging from environmental monitoring \cite{r1}, positioning service \cite{r2_2}, to disaster rescue \cite{r2}. It requires robots to observe events in a given region, which is often described by a density function, and then self-organize into spatial configurations that optimize sensing quality metrics according to the importance of these events.

A move-to-centroid strategy using Lloyd’s algorithm is introduced in \cite{r3}, where the agents are driven towards the center of their respective Voronoi regions to optimize the coverage cost function. 
To deal with scenarios where the density function changes over time, a continuous-time coverage algorithm is proposed by Lee et al. in \cite{r4} to lend robots to a distributed implementation while guarantee the optimization.
In addition, by taking into account the sensing ranges of the robots, Ref. \cite{r5} comes up with a distributed algorithm for sensing-limited robots using a combination of ground and aerial robots. Furthermore, Voronoi partitions also have been used in persistent coverage algorithms\cite{r5_1}. However, these algorithms are restricted to convex environments because the Voronoi partitions do not address the shape of the environment nor the obstacles\cite{r6}, which are more common in reality.

\begin{figure}[!tb] 
    \centering 
    \includegraphics[width=0.9\linewidth]{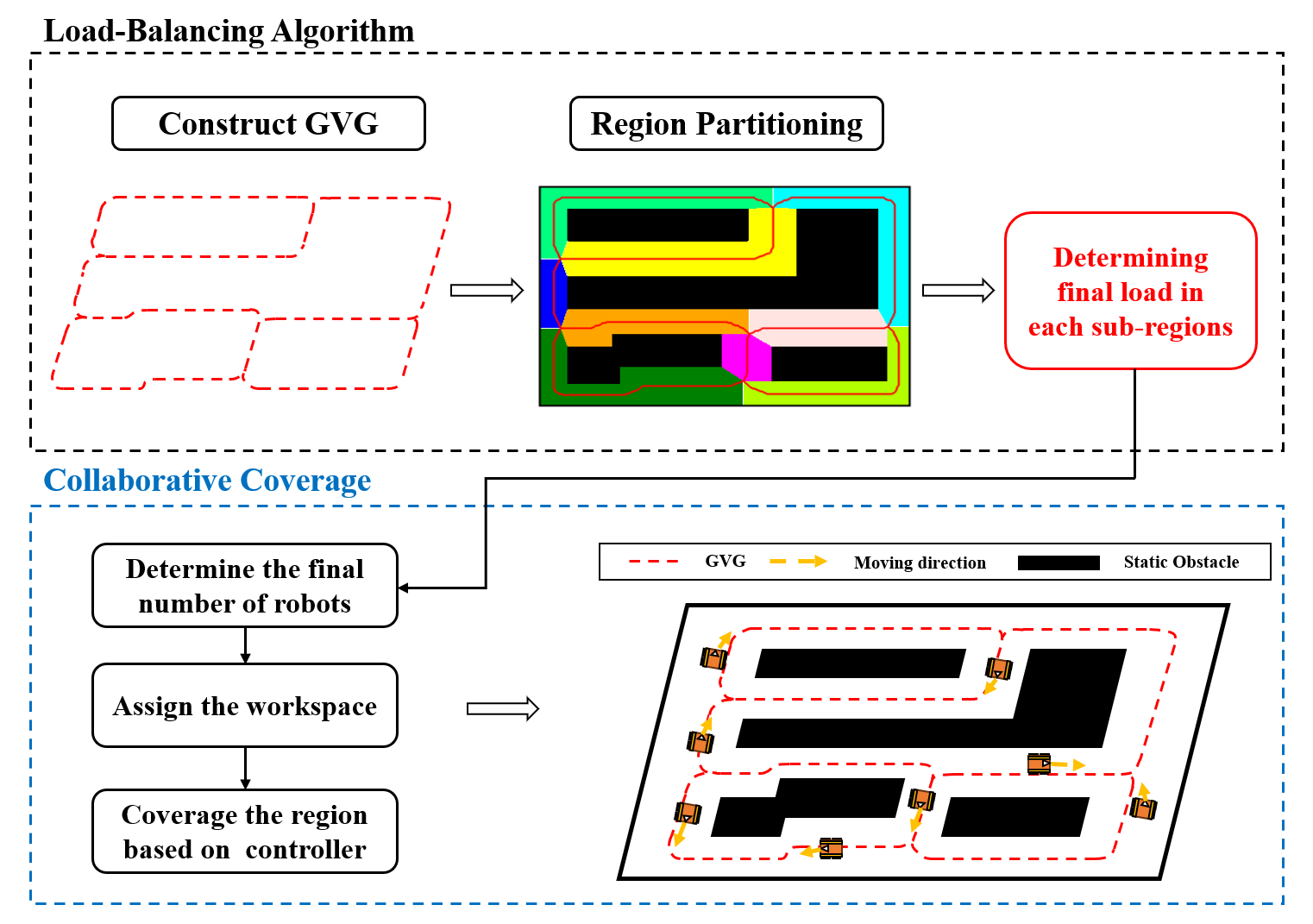} 
    \caption{Overview of the proposed method, which includes Load-Balancing Algorithm phase (top) and Collaborative Coverage phase (bottom). Load-Balancing Algorithm determines the final number of robots in each region. In Collaborative Coverage phase, each robot are controlled to efficiently cover its region based on the designed controller.} 
    \label{fig:illus} 
\end{figure}

To address this limitation, some algorithms are proposed to control robots to achieve coverage in non-convex environments. In \cite{r7} and \cite{r8}, the environment is transformed to a convex region using a diffeomorphism. But due to high computational cost and the possibility of leading to a different implementation from that in the original environment, the coverage performance is greatly limited. Ref.~\cite{r9} presents a strategy that combines Lloyd’s algorithm and TangentBug to provide Voronoi coverage in non-convex environment, but it may lead to a result that coverage effect may be suboptimal. 
\cite{r9_1} introduces a method using boosting functions to escape local minimum for multi-robot systems.
Since the workload of the divided regions is not always equal when Voronoi partitions are performed, an equitable workload partition is proposed by dividing the whole region, which is convex polygon or with parallel boundaries, into multiple stripes \cite{r10,r11}. To cope with non-convex scenarios, a sectorial coverage control algorithms for load balancing as well as optimization of coverage performance has been introduced \cite{r12}. However, it can only be applied to star-shaped regions and fail to deal with regions with multiple holes/obstacles. In fact, such cases (i.e., regions with multiple holes/obstacles) are far more common in practical scenarios, making this limitation a critical constraint on its applicability.

In order to extend the coverage to non-convex environments \blue{with multiple static obstacles} while considering load-balancing, this paper proposes a novel coverage control approach for dividing and covering regions using generalized Voronoi graph (GVG) \cite{r13,r14,r15}, as is shown in Fig.~\ref{fig:illus}. Due to its capability to characterize the topological structure of the entire environment, GVG has been broadly applied to robotic navigation tasks. In this paper, we use GVG to divide the whole region into sub-regions. Unlike load-balancing algorithms in \cite{r17,r18}, which only consider the loads of edges while ignoring the difference of edge weights, our algorithm takes the weights into account, whose design addresses the core challenge that cannot be resolved by the methods above. In summary, the main contributions of this work are listed as follows.
\begin{enumerate}
\item[1)] Propose a novel coverage strategy in non-convex environments with multiple static obstacles by using GVG as a basis for assigning and guiding robots. 
\item[2)] Design a novel distributed load-balancing algorithm, which is mainly intended to address the situation where edge weights need to be considered.
\item[3)] Provide detailed proof of convergence for the load-balancing algorithm and coverage control. Simulations are conducted to validate the effectiveness.
\end{enumerate}

The rest of this paper is organized as follows. Section~\ref{sec:2} formulates the coverage problem. Section~\ref{sec:3} introduces load-balancing algorithm for the assignment of the robots. Section~\ref{sec:4} provides distributed coverage algorithm. Section~\ref{sec:5} presents experimental validation of the algorithms. Section~\ref{sec:6} concludes this paper and demonstrates our future work.

\section{PROBLEM STATEMENT}
\label{sec:2}
\subsection{GVG}

Consider a non-convex area $D \subset \mathbb{R}^2$ with several holes which represent static obstacles $C_1, \ldots, C_{M}$. We denote $C_0$ as the external boundary of the area, which belongs to the obstacle set $\{C_i\}$. The distance function $d_i(q)$ is formulated as $d_i(q)=\min_{c \in C_i} \| q - c \|$, which measures the minimum distance from point $q$ to static obstacle $C_i$. The gradient of $d_i(q)$ is $\nabla d_i(q)=\frac{q-c_0}{\|q-c_0\|}$, where $\nabla d_i(q)$ is a unit vector in the direction from $q$ to the nearest point $c_0$ in $C_i$. Then the set of GVG edge defined by $C_i$ and $C_j$ is
\begin{equation*}
\begin{aligned}
E_{ij} = \{ q \in D: d_i(q) = d_j(q) \leq d_k(q) \forall k \\
\text{such that } 
\nabla d_i(q) \neq \nabla d_j(q) \}.
\end{aligned}
\end{equation*}
Let $|E|$ denote the cardinality of the set $\{E_{ij}\}$. The points where a GVG edge $E_{ij}$ terminate are defined as \emph{nodes}, which have the properties that $d_i(q) = d_j(q) = d_k(q) $ for at least one $k \neq i,j$. A circle, referred to as the \emph{node circle}, is formed by taking the node as the center and the minimum distance from the node to the nearest obstacle as the radius, where the points at intersection are called the \emph{closest points}.
From Fig.~\ref{fig:1}, it is obvious that this circle maintains a tangential relationship with several obstacles, yet no part of its inner area comes into contact with any obstacle. 

\begin{figure}[t] 
    \centering 
    \includegraphics[width=0.75\linewidth]{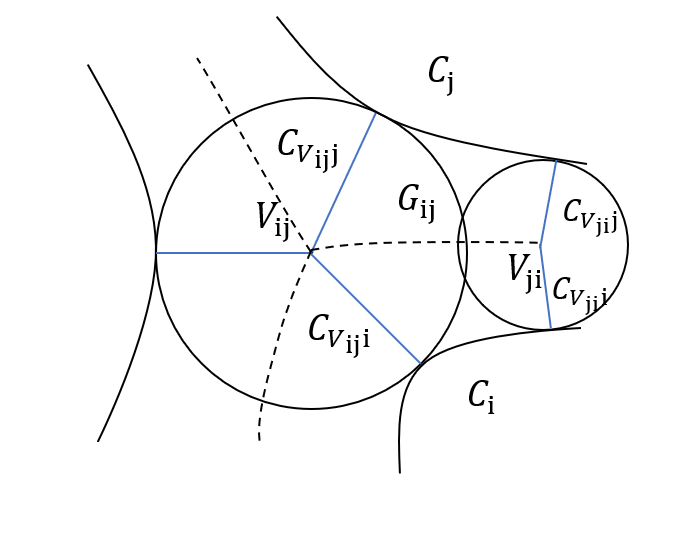} 
    \caption{$V_{ij}$ and $V_{ji}$ are two GVG nodes defined by $E_{ij}$. The circle is the node circle. The blue line represents the connection between the node and the intersection point of the circle and an obstacle, and it also serves as the boundary of the region.} 
    \label{fig:1} 
\end{figure}

Consider a GVG edge $E_{ij}$. For convenience in definition, we denote the nodes on its two sides as $V_{ij}$ and $V_{ji}$, respectively. Let \blue{$C_{V_{ij}i}$} be the line segment between $V_{ij}$ and the closest point in $C_i$. We thus define a region in $D$ that is bounded by $C_{V_{ij}i}$,$C_{V_{ij}j}$,$C_{V_{ji}i}$,$C_{V_{ji}j}$ and the boundary of the obstacles $\partial C_i$ and $\partial C_j$ as a GVG cell $G_{ij}$, which is shown in Fig.~\ref{fig:1}.

\subsection{PROBLEM FORMULATION}

Let $\mathcal{K} = \{1,\ldots,K\}$ denote the set of robots trying to coverage the area $D$ with positions $p_i(t) \in D, i \in \mathcal{K}$. The density value of any $q\in D$ is $\phi(q)$. Consider the circumstance that $K$ is large enough. Define the matrix representing the index of each cell as $\mathcal{C}=\{1,\ldots,|E|\}$ and each GVG cell $G^i,i\in \mathcal{C}$ (for the sake of argumentative convenience, we will omit the subscripts and use superscripts to denote the indices of cells in the subsequent discussion, i.e. $G^i$ and $E^i$ ) contains $K_i$ robots, where $K_i$ is at least one. In each GVG cell, every robots are assigned a workspace $O_{j},j\in \mathcal{K}$. So the purpose of coverage control is to minimize the coverage cost function
\begin{equation}
\begin{aligned}
\mathcal{H}=\sum_{i = 1}^{|E|}\sum_{j = 1}^{K_i}\int_{O_j}f(p_j,q)\phi(q)\, dq . \label{e1}
\end{aligned}
\end{equation}
The dynamics of the $j$th robots is presented as
\begin{equation}
\begin{aligned}
\dot{p}_j(t) = u_j, j \in \mathcal{K}, \label{e2}
\end{aligned}
\end{equation}
where $u_j$ is the control input. In general, the aim of this paper is to design a distributed control algorithm to minimize the cost function~\eqref{e1}.

\section{LOAD-BALANCING ALGORITHM}
\label{sec:3}

In this section, we provide a distributed load-balancing algorithm that takes weights into account. We denote the mass of the GVG cell $G^i$ as $e_i=\int_{G^{i}} \phi(q) \, dq,i\in \mathcal{C}$. Therefore the total mass of the area is $e= \sum\nolimits_{i=1}^{|E|} e_i$. To achieve a balanced workload to enhance the coverage performance, robots are controlled according to the mass of GVG cell to evenly distribute within the area. This strategy allows the robots to better cover the whole area while consider the influence of the mass of GVG cell, which we call a \emph{weighted quantization consensus problem}.

Assume that each robot within a given cell is aware of the number of robots in that cell and the cell's mass, as well as the \blue{same information} from its neighboring cells. The set of neighbors of cell $G^i$ is denoted as  $\mathcal{N}^i$ to represent the cells that are adjacent to $G^i$. Then we use $x_j(t)=\frac{K_j(t)}{e_j},\forall j \in \mathcal{C}$ to express the load of the cell, where $t \ge0$ represents the execution time.

At first, every robots in each cell $G^i$ run the algorithm \ref{alg:1} to determine the final ideal number of robots for cell $G^i$. They calculate the load in their own cell and then compare it with the minimum load among their neighbors(Lines 1-3). When the neighbor's load is smaller, they take the average of the two loads(Lines 4-7). It is noted that the calculated load $x_i(t)$ in the algorithm is not necessarily an integer. Theorem~\ref{alg:1} proves the asymptotic convergence of the $x$-values of all cells to an identical value as $t \rightarrow \infty$. For practical implementation, we select a sufficiently large finite time $t_1$ and adopt the $x$-values at $t_1$ as the approximate ideal load for each cell to proceed with the load-balancing process.     

For cell $G^i$, its ideal number of robots $K_i^*$ is obtained after Algorithm~\ref{alg:1} completes execution. For each cell $G^i$, we define $c_i = K_i - \lfloor K_i^* \rfloor$, which evaluates the deviation between the actual number of robots $K_i$ and $\lfloor K_i^*\rfloor$, where $\lfloor \cdot\rfloor$ is the floor operator. In this way, we transform a weighted quantization consensus problem into an unweighted one that solely considers integer values\cite{r17,r18}. From this, it can be seen that the values of $c_i$ can be either positive or negative. 
We define $\bar{c}= \frac{1}{|E|}\sum\nolimits_{i=1}^{|E|} c_i \in[0,1)$. \blue{In order to achieve load balance while meeting the integer constraint, we will introduce the following definition.}

\begin{definition}
\label{le:3.1}
A necessary condition for load balancing is that, for each cell $G^i$, its value of $c_i$ will eventually be either $c_i = \lfloor\bar c \rfloor$ or $c_i = \lceil\bar c\rceil$, where $\lceil \cdot \rceil$ is the ceiling operator. 
\end{definition}

Therefore, according to Definition \ref{le:3.1}, when load balancing is achieved, it must hold that
\begin{equation}
\left\{
\begin{gathered}
\alpha + \beta = |E|, \\ 
\alpha \lfloor\bar c\rfloor +\beta \lceil\bar c\rceil = |E| \bar c, 
\end{gathered}\label{e3}
 \right.
\end{equation}
where $\alpha$ and $\beta$ are the number of cells at $\lfloor\bar c\rfloor$ and $\lceil\bar c\rceil$, respectively. 
Based on this, we propose Algorithm \ref{alg:2}.

Algorithm \ref{alg:2} can be divided into three phases: Offering Phase(Lines 3-7), Accepting Phase(Lines 8-13), and Passing Phase(Lines 14-19). Initially, cell $G^i$ randomly selects a neighbor $j_o$ whose $c_{j_o}$ is the minimum. If $c_{j_o}<c_i$, it sends offer to $j_o$. Then in Accepting Phase, the cell $i$ will check the offers it has received, which are stored in $R_1 = \{(j_1,c_{j_1}(t),\ldots,(j_r,c_{j_r}(t)\}$. If it has received offers, i.e. $|R_1|>0$, then it chooses $j_a$ randomly from neighbors that has the largest value $c_{j_a}=\max_{s \in \{1,\ldots,r\}} c_{s}$ and sends acceptance. Lastly, in Passing Phase, the cell $G^i$ will check its acceptance $R_2=\{h\}$. If it has acceptance, i.e. $|R_2|>0$, it will pass a robot to the cell $h$.

\begin{algorithm}[t]
\caption{ Determining ideal number of robots for cell $i$} 
\label{alg:1}
\begin{algorithmic}[1] 
\Statex \textbf{Input:} $x_j(0),\forall j \in \mathcal{N}^i,x_i(0),e_i,t_1$. 
\Statex \textbf{Output:} $K_i^*.$
\State $t \leftarrow 0$;
\While {$t<t_1$}
  \State Select one neighbor $j$ with minimum load $x_j(t)$;
  \If{$x_j(t)<x_i(t)$}
    \State $\Delta x \leftarrow \frac{x_i(t)-x_{j}(t)}{2};$
    \State $x_i(t+1)\leftarrow x_i(t)-\Delta x;x_j(t+1)\leftarrow x_j(t)+\Delta x;$
  \EndIf
  \State Increment $t$;
\EndWhile
\State $ K_i^* \leftarrow e_ix_i(t_1)$;
\end{algorithmic}
\end{algorithm}

\begin{algorithm}[t]
\caption{ Calculate final number of robots for cell $i$} 
\label{alg:2}
\begin{algorithmic}[1] 
\Statex \textbf{Input:} $c_j(t),\forall j \in \mathcal{N}^i\cup\{i\},t_2,t_1$. 
\State $t \leftarrow t_1;$
\While {$t<t_2$}
  \State {\footnotesize\textbf{// Offering Phase:} }
  \State Select a neighbor $j_o$ with minimum $c_{j_o}(t)$;
  \If{$c_{j_o}(t)<c_{i}(t)$}
    \State Send offer to $j_o$ consisting of ($i,c_{i}(t)$);
  \EndIf
  \State {\footnotesize\textbf{// Accepting Phase:} }
  \State Check the received offer $R_1$;
  \If{$|R_1|>0$}
    \State Select a neighbor $j_a$ with maximum $c_{j_a}(t)$;
    \State Send acceptance to $j_a$ consisting of (i);
  \EndIf
  \State {\footnotesize\textbf{// Passing Phase:} }
  \State Check the acceptance $R_2=\{(h)\}$;
  \If{$|R_2|>0$}
    \State $\delta \leftarrow 1$;
    \State $c_{i}(t)\leftarrow c_{i}(t)-\delta;c_{h}(t)\leftarrow c_{h}(t)+\delta $;
  \EndIf
  \State Increment $t$;
\EndWhile
\State $K_i \leftarrow \lfloor K_i^*\rfloor +c_{i}(t_2)$;
\end{algorithmic}
\end{algorithm}

Then we will concentrate on the analysis of the Load-Balancing Algorithm.

\begin{theorem}
\label{theo:1}
In the execution of Algorithms \ref{alg:1} and \ref{alg:2}, Eq.~\eqref{e3} is satisfied as $t \rightarrow \infty$.
\end{theorem}

\begin{proof}
The convergence proof of Algorithm \ref{alg:1} is similar to \cite[Claim 5.8]{r17}. From \cite[Theorem~1]{r18}, after the execution of Algorithm \ref{alg:2}, for any $c_i(t),i \in \mathcal{C}$, it holds that $\lim_{t \to \infty} \Pr[c_i(t) \in \mathcal{T}] = 1$, where $\mathcal{T}=\{c \mid c_i \in \{\lfloor \bar c\rfloor,\lfloor \bar c\rfloor + 1\}, \, i = 1, \ldots, |E|\}$, which completes the proof.
\end{proof}

Theorem \ref{theo:1} ensures the final configuration of the number of robots in each cell, and is a necessary condition for the balance of loads. However, we note that this theorem only constrains the values of $\alpha$ and $\beta$, and does not specify the allocation of cells corresponding to the two cases of $\alpha$ and $\beta$. Noting that the load of each cell is a crucial factor in determining load balance, we thus need to further consider the impact of different cell allocation situations on load balance.

Define the fraction part of the number $\left\{ x \right\} = x - \lfloor x \rfloor$. Let $\mathfrak{M} = (\{K^*_{i_1}\},\{K^*_{i_2}\},\ldots,\{K^*_{i_|E|}\})$ be an ordered set, where $\{K^*_{i_1}\} \leq \{K^*_{i_2}\} \leq \ldots \leq \{K^*_{i_|E|}\}$. Let $S_1$ be the sum of the elements in $\mathfrak{M}$. If the load-balancing function is defined as $\mathcal{S}=\sum\nolimits_{i=1}^{|E|} |c_i-K^*_i|$, then we can introduce Definition \ref{le:1} and Theorem \ref{theo:2}.

\begin{definition}
\label{le:1}
For the ideal configuration, it must hold that for the first $\alpha$ elements of the ordered set $\mathfrak{M}$, the corresponding cells have $c=\lfloor \bar c\rfloor$, while the remaining elements correspond to cells with $c=\lceil \bar c\rceil$.
\end{definition}

\begin{theorem}
\label{theo:2}
For the final configuration, there is $\mathcal{S}_p\leq\mathcal{S}_f<\mathcal{S}_p+S_2$, where $S_2=\min(|E|-S_1,S_1)$. And $\mathcal{S}_p$ and $\mathcal{S}_f$ are the value of load-balancing function for the ideal configuration and the final configuration, respectively.
\end{theorem}
\begin{proof}
As system reaches the final configuration, $\bar c=\frac{S_1}{|E|}\in[0,1)$, we have $\alpha=|E|(1-\bar c)=|E|-S_1,\beta =|E|\bar c=S_1.$ When the final configuration aligns with the ideal configuration, according to Definition \ref{le:1}, $\mathcal{S}_f=\mathcal{S}_p=\{K^*_{i_1}\}+\{K^*_{i_2}\}+\dotsb+\{K^*_{i_{\alpha}}\}+\beta-\{K^*_{i_{\alpha+1}}\}-\{K^*_{i_{\alpha+2}}\}-\dotsb-\{K^*_{i_{|E|}}\}$. Following this, we consider the most deviant configuration, which holds the fact that for the first $\beta$ elements in $\mathfrak{M}$, the corresponding cells have $c=\lceil \bar c\rceil$, while the remaining elements correspond to cells with $c=\lfloor \bar c\rfloor$. Thus, $\mathcal{S}_f'=\beta-\{K^*_{i_1}\}-\{K^*_{i_2}\}-\dotsb-\{K^*_{i_\beta}\}+\{K^*_{i_{\beta+1}}\}+\{K^*_{i_{\beta+2}}\}+\dotsb+\{K^*_{i_{|E|}}\}=\mathcal{S}_p+\{K^*_{i_{|E|}}\}+\dotsb+\{K^*_{i_{|E|-S_2}}\}-\{K^*_{i_{1}}\}-\dotsb-\{K^*_{i_{1+S_2}}\}\leq \mathcal{S}_p+S_2(\{K^*_{i_{|E|}}\}-\{K^*_{i_{1}}\})<\mathcal{S}_p+S_2$, which completes the proof.
\end{proof} 

Theorem~\ref{theo:2} quantifies the bounded deviation between the actual configuration of robots in GVG cells and the theoretically ideal configuration. This bound clarifies the maximum possible deviation of the actual configuration from the ideal one, establishing a rigorous performance guarantee for the proposed load-balancing algorithm.

\section{COLLABORATIVE COVERAGE}

\label{sec:4}

After the execution of Algorithms~\ref{alg:1} and \ref{alg:2},
we can strictly limit the final robot configuration to the ceiling or floor values of the ideal configuration for robots on each cell. Furthermore, to further enhance the coverage performance, we consider the robot coverage problem in each cell $G^{i} \subset D,i\in \mathcal{C}$. The set of robots in this cell is denoted by $\mathcal{K}_{i}=\{1,\ldots,K_{i} \}$.


In cell $G^i$, we denote the set of points with the property by \( Q^i(p) \), that is, \( Q^i(p) = \left\{ q \notin E^i \mid d(q, p) < d(q, p'), \forall p' \in E,p' \neq p\right\} \), where $p$ is a point on the GVG edge $E^i$. Therefore, we give the following lemma.

\begin{lemma}
\label{le:2}
 For any point $p$ on the GVG edge $E^i$, \( Q^i(p) \cup\{p\} \) forms a continuous one-dimensional path embedded in $\mathbb{R}^2$. Furthermore, for all points on $E^i$, \( \bigcup_{p \in E^i} Q^i(p) = G^{i} \setminus E^i \). For any \( p_j \neq p_k \) in $E^i$, \( Q^i(p_j) \cap Q^i(p_k) = \varnothing \).
\end{lemma}

\begin{proof}
We consider a sufficiently small neighborhood $U$ of $p$ in which any point $q$ is extremely close to $p$. Therefore, we establish a local coordinate system \((s, r)\) at point \(p\), where the \(s\) - axis is along the direction of the GVG edge $E^i$ and the \(r\) - axis is perpendicular to $E^i$. For any point \(q=(s, r)\) and a neighboring point \(p'=(\Delta s, 0) \in E^i\), we have:
\[
d(q, p) = \sqrt{s^2 + r^2}, \quad d(q, p') = \sqrt{(s - \Delta s)^2 + r^2}.
\]
According to the definition of \(Q^i\), we require \(d(q, p') > d(q, p)\), that is:$\sqrt{(s - \Delta s)^2 + r^2} > \sqrt{s^2 + r^2}$. Then we have $s < \frac{\Delta s}{2}$. As \(\Delta s \to 0\), the solution set converges to \(s = 0\), i.e., a line perpendicular to $E^i$. Then we prove the first property.

For any \( q \in G^{i} \setminus E^i \), there exists a unique closest point \( p \in E^i \). Therefore, \( q \) must belong to some \( Q^i(p) \), implying \( \bigcup_{p \in E^i} Q^i(p) = G^{i} \setminus E^i \).

Assume for contradiction that there exists a point \( q \in Q^i(p_j) \cap Q^i(p_k) \). Then \( q \) must satisfy \( d(q, p') > d(q, p_j) \) for all \( p' \in E^i \setminus \{p_j\} \) and \( d(q, p') > d(q, p_k) \) for all \( p' \in E^i \setminus \{p_k\} \). Substituting \( p' = p_k \) into the first condition gives \( d(q, p_k) > d(q, p_j) \), and substituting \( p' = p_j \) into the second gives \( d(q, p_j) > d(q, p_k) \), which is a contradiction. Thus, \( Q^i(p_j) \cap Q^i(p_k) = \varnothing \) for any \( p_j \neq p_k \).
\end{proof}

Therefore, according to Lemma \ref{le:2}, all the points $q$ in the GVG cell $G^i \setminus E^i$ can be projected onto $E^i$, and the projection point is denoted as $q^i(t),i\in \mathcal{C}$, which is closest to the $q$. According to the definition of GVG cell, the boundaries of multiple GVG cells are established by inscribed circles centered on nodes, as shown in Fig.~\ref{fig:1}.
The GVG edge $E^i$ is represented as a piecewise-smooth parametric curve $\gamma^i : [0, L^i] \to \mathbb{R}^2,i\in \mathcal{C}$, which is twice continuously differentiable and $L^i$ is the total length of the edge $E^i$. The projection point of the $j$th robot in $G^i$ is $p_j^i(t),j\in\mathcal{K}_i$, whose corresponding parameter on the curve is $s^i_j\in [0,L^i]$, i.e. $\gamma^i(s^i_j(t))=p_j^i(t)$. For each point $\gamma^i(s),s\in[0,L^i]$ on the GVG edge $E^i$, we define the Frenet–Serret coordinate system $(s,r)$ as follows
\begin{equation*}
\left\{
\begin{aligned}
&\boldsymbol{\tau}(s)=\frac{d\gamma^i}{ds}, \\
&\boldsymbol{v}(s)= R·\boldsymbol{\tau}(s),
\end{aligned}
\right.
\end{equation*}
where $\boldsymbol{\tau}(s)$ is the unit tangent vector at $\gamma^i(s)$ and 
$R=\begin{bmatrix} 0 & 1 \\ -1 & 0 \end{bmatrix}$
is the rotation matrix. Thus each point $q\in G^i$ can be represented by
\begin{equation}
q(s,r)=\gamma^i(s)+r\boldsymbol{v}(s),s\in[0,L^i],r\in[-\epsilon(s),\epsilon(s)],
\label{e:q}
\end{equation}
where $\epsilon(s)$ is the distance from point $\gamma^i(s)$ to the closest obstacle along the direction of $\boldsymbol{v}(s)$.

Assume that each robot is equipped with a virtual partition boundary, which divides the whole cell into multiple sub-regions. Then the boundary between the $j$-th and $j+1$-th robots are $\varphi_j^i=Q^i(b_j^i)\cup \{b_j^i\}$, where $b_j^i=\gamma^i(\frac{s_j^i(t)+s_{j+1}^i(t)}{2})$ and $1\leq j \leq K_i-1$. The first and last sub-regions are determined by the region boundaries, which are defined as $\varphi_0^i$ and $\varphi_{K_i}^i$, respectively. Let ${\varphi_j^{-1}}$ denote the scalar arc-length parameter of $\gamma^i(s)$ at the intersection of the boundary $\varphi_j^i$ and $\gamma^i$. In particular, the position of the $j$-th robot can be represented as $p_j=p_j^i+\delta_j\boldsymbol{v}(s_j^i)$, where $\delta_j$ is the offset of the robot in the normal direction of $\gamma^i$.

Consider the cost function in cell $G^i$ to be minimized
\begin{equation}
\mathcal{H}^i=\sum_{j = 1}^{K_i}\int_{O_j}f(p_j,q)\phi(q)\, dq  .
    \label{e4}
\end{equation}
Then we provide the following theorem.

\begin{theorem}
\label{theo:3}
For the GVG cell $G^i$, when $f(\cdot)=\|q-p_i\|^2$, the control law
\begin{equation}
    \begin{gathered}
        u_j= -k_{g} \int_{\varphi_{j-1}^{-1}}^{\varphi_{j}^{-1}} \bigl(p_j^i - \gamma^i(s)\bigr) \hat{\phi}(s) \, ds 
        + k_{g} \\
    \int_{\varphi_{j-1}^{-1}}^{\varphi_j^{-1}} \int_{-\epsilon(s)}^{\epsilon(s)}(r\boldsymbol{v}(s)-\delta_j\boldsymbol{v}(s_j^i))\cdot\phi\cdot(1-r\kappa(s))drds
        \label{e5}
    \end{gathered}
\end{equation}
where  $\hat{\phi}(s)=\int_{-\epsilon(s)}^{\epsilon(s)}\phi\bigl(\gamma^i(s) + r \boldsymbol{v}(s)\bigr)\bigl(1 - r \kappa(s)\bigr)dr$, $\kappa(s)$ is the curvature of $\gamma^i$ and $k_{g}$ is the positive gain, minimizes the cost function \eqref{e4} in a gradient-descent manner and allows the robots to perform coverage along the GVG edge.
\end{theorem}

\begin{proof}
From Eq.~\eqref{e:q}, we have transformed $q\in\mathbb{R}^2$ into its representation in the coordinate system $(s,r)$. Therefore, we have $dq=J(s,r)dsdr$, where $J(s,r)$ is the Jacobian determinant. Given that 
\begin{equation*}
{\frac{\partial q}{\partial s}}=\boldsymbol{\tau}(s)(1-r\kappa(s)), \quad {\frac{\partial q}{\partial r}}=\boldsymbol{v}(s),
\end{equation*}
$J(s,r)$ can be computed by
\begin{equation*}
    \begin{gathered}
    J(s,r)=\left| \begin{matrix} {\boldsymbol{\tau}_x-r\kappa\boldsymbol{\tau}_x}&{\boldsymbol{v}_x} \\ {\boldsymbol{\tau}_y-r\kappa\boldsymbol{\tau}_y} & {\boldsymbol{v}_y} \end{matrix} \right| = (1-r\kappa)(\boldsymbol{\tau}_x\boldsymbol{v}_y-\boldsymbol{\tau}_y\boldsymbol{v}_x).
    \end{gathered}
\end{equation*}
Since the two vectors $\boldsymbol{\tau}$ and $\boldsymbol{v}$ are unit orthogonal vectors, $\boldsymbol{\tau}_x\boldsymbol{v}_y-\boldsymbol{\tau}_y\boldsymbol{v}_x=1$, so $J(s,r)=1-r\kappa(s)$.

Thus, the cost function~\eqref{e4} can be written as
\begin{equation}
\begin{gathered}
\mathcal{H}^i=\sum_{j = 1}^{K_i}\int_{\varphi_{j-1}^{-1}}^{\varphi_j^{-1}}\int_{-\epsilon(s)}^{\epsilon(s)}\|q-p_j\|^2\phi(q)(1-r\kappa(s))\, drds  .
\label{e6}
\end{gathered}
\end{equation}
We split Eq.~\eqref{e6} into two parts, i.e.
\begin{equation}
    \begin{gathered}
        \mathcal{H}^i=\mathcal{H}^i_{tan}+\mathcal{H}^i_{norm},
        \label{e7}
    \end{gathered}
\end{equation}
where $\mathcal{H}^i_{tan}$ and $\mathcal{H}^i_{norm}$ are the decomposition of the cost function in the tangential direction and normal direction of the GVG curve $\gamma^i$, respectively. Let the projection of density in the normal direction be $\hat{\phi}(s)=\int_{-\epsilon(s)}^{\epsilon(s)}\phi\bigl(q(s,r))\bigl(1 - r \kappa(s)\bigr)dr$ and we could get
\begin{equation*}
\begin{gathered}
\mathcal{H}^i_{tan}=\sum_{j = 1}^{K_i}\int_{\varphi_{j-1}^{-1}}^{\varphi_j^{-1}}\|\gamma(s)-p_j^i\|^2\hat{\phi}(s)\, ds  .
\end{gathered}
\end{equation*}
$\mathcal{H}^i_{norm}$ can be derived as
\begin{equation*}
\begin{gathered}
\mathcal{H}^i_{norm}=\mathcal{H}^i-\mathcal{H}^i_{tan}\\
=\sum_{j = 1}^{K_i}\int_{\varphi_{j-1}^{-1}}^{\varphi_j^{-1}}\int_{-\epsilon(s)}^{\epsilon(s)}(\|\gamma(s)+r\boldsymbol{v}(s)-p_j\|^2-
\\\|\gamma(s)-p_j^i\|^2)\phi(\gamma(s)+r\boldsymbol{v}(s))(1-r\kappa(s))\, drds.
\end{gathered}
\end{equation*}

Taking the partial derivative of $\mathcal{H}^i_{tan}$ and $\mathcal{H}^i_{norm}$, we could get
\begin{equation*}
    \begin{gathered}
        \frac{\partial\mathcal{H}^i_{tan}}{\partial p_j}=2\int_{\varphi_{j-1}^{-1}}^{\varphi_j^{-1}}(p_j^i-\gamma(s))\hat{\phi}(s)\, ds
    \end{gathered}
\end{equation*}
and
\begin{equation*}
    \begin{gathered}
        \frac{\partial\mathcal{H}^i_{norm}}{\partial p_j}=\int_{\varphi_{j-1}^{-1}}^{\varphi_j^{-1}}\int_{-\epsilon(s)}^{\epsilon(s)}2[(p_j-\gamma(s)-r\boldsymbol{v}(s))-\\(p_j^i-\gamma(s))]\phi(\gamma(s)+r\boldsymbol{v}(s))(1-r\kappa(s))drds\\=-2\int_{\varphi_{j-1}^{-1}}^{\varphi_j^{-1}} \int_{-\epsilon(s)}^{\epsilon(s)}(r\boldsymbol{v}(s)- \delta_j\boldsymbol{v}(s_j^i))\cdot\phi\cdot(1-r\kappa(s))drds.
    \end{gathered}
\end{equation*}

Therefore, to minimize \eqref{e4}, it must hold that $\frac{\partial\mathcal{H}^i}{\partial p_j}=\frac{\partial\mathcal{H}_{tan}^i}{\partial p_j}+\frac{\partial\mathcal{H}_{norm}^i}{\partial p_j}=0$. Then we let the control input be $u_j = -k_g'\frac{\partial\mathcal{H}^i}{\partial p_j}$, where $k_g'$ is a positive gain, which completes the proof. 
\end{proof}

Theorem \ref{theo:3} indicates that the control input \eqref{e5} constrains the robot to move along the distribution of the GVG curve $\gamma^i$ while minimizing the cost function \eqref{e4}. 

Furthermore, note that the mass center along the GVG curve $\gamma^i$ and the mass of $\gamma^i$ in region $O_j$ are $p_{j,tan}^i=\frac{\int_{\varphi_{j-1}^{-1}}^{\varphi_j^{-1}}\gamma^i(s)\hat{\phi}(s)ds}{\int_{\varphi_{j-1}^{-1}}^{\varphi_j^{-1}}\hat{\phi}(s)ds}$ and $M_{j,tan}^i=\int_{\varphi_{j-1}^{-1}}^{\varphi_j^{-1}}\hat{\phi}(s)ds$, respectively. In this case,
\begin{equation}
    \begin{gathered}
         \frac{\partial\mathcal{H}^i_{tan}}{\partial p_j}=2\int_{\varphi_{j-1}^{-1}}^{\varphi_j^{-1}}(p_j^i-\gamma(s))\hat{\phi}(s)\, ds \\
         =2M_{j,tan}^i(p_j^i-p_{j,tan}).
    \end{gathered}
    \label{eq:tan}
\end{equation}
Similarly, when the controller \eqref{eq:tan} converges to $p_{j,tan}^i$, $\frac{\partial\mathcal{H}^i_{norm}}{\partial p_j}$ can be written as
\begin{equation}
    \begin{gathered}
         \frac{\partial\mathcal{H}^i_{norm}}{\partial p_j}
         =2M_{j,norm}^i(\delta_j-p_{j,norm})\boldsymbol{v}(p_{j,norm}),
    \end{gathered}
    \label{eq:norm}
\end{equation}
where $p_{j,norm}=\frac{\int_{-\epsilon(p_{j,tan})}^{\epsilon(p_{j,tan})}r\cdot\phi\cdot(1-r\kappa(p_{j,tan}))dr}{\int_{-\epsilon(p_{j,tan})}^{\epsilon(p_{j,tan})}r\cdot\phi\cdot(1-r\kappa(p_{j,tan}))dr}$ and $M_{j,norm}=\int_{-\epsilon(p_{j,tan})}^{\epsilon(p_{j,tan})}\phi(\gamma^i(p_{j,tan})+r\boldsymbol{v}(p_{j,tan}))(1-r\kappa(p_{j,tan}))dr$. Therefore, controller \eqref{e5} can be rewritten as
\begin{equation}
    \begin{aligned}
         u_j=-k_g[M_{j,tan}^i(p_j^i-p_{j,tan})+\\
          M_{j,norm}^i(\delta_j-p_{j,norm})\boldsymbol{v}(p_{j,norm})].
    \end{aligned}
    \label{eq:rewritten}
\end{equation}
Controller~\eqref{eq:rewritten} consists of two parts: first, controlling the robot to move along the GVG curve $\gamma^i$ to reach its mass center, and then moving along the normal direction to reach the centroid at the normal position. Thus, the robots are able to efficiently cover the whole region while achieving a high coverage quality.

\section{EXPERIMENTAL VALIDATION}
\label{sec:5}

In this section, we will provide validation of the proposed algorithm through simulations. In the simulations, the field of the region $D$ is chosen to be a $372\times 247$ rectangle area with four holes, which is shown in Fig.~\ref{fig:2}. According to the cell partition method shown in Section \ref{sec:2}, the whole area is divided into 9 cells. The density value is set to $\phi(x,y)=10^{-8}[(x-186)^2+(y-86)^2]$. In order to validate the effectiveness of the proposed scheme, we conduct the simulation using a group of 20 robots. 

The robots are first randomly assigned to the region, represented by green triangles as shown in Fig.~\ref{fig:5a}. Then Algorithm \ref{alg:1} and Algorithm \ref{alg:2} are performed to calculate the final configuration of the robots, which is illustrated in Fig.~\ref{fig:3}. It can be observed that for some $i\in \mathcal{C}$, $K_i$ vibrates continuously between two adjacent integer values. There are primarily two reasons underlying this issue. One is that the ideal number of robots to be allocated in these nine regions is not all integers. The other is that, given the distributed structure of the entire system, it is impossible to determine whether the system is in the ideal configuration, which leads to the continuous execution of the algorithms. Hence, we stop the execution after 80 iterations and each cell has obtained the ideal number of robots and the actual number of robots: $K^*_1=2.33,K^*_2=3.39,K^*_3=1.20,K^*_4=1.12,K^*_5=1.41,K^*_6=1.57,K^*_7=0.50,K^*_8=4.77,K^*_9=3.70$ and $K_i = 3, i=1,2,9,K_i=1,i=3,4,5,7 ,K_{6}=2$ and $K_{8}=5$.

\begin{figure}[htbp] 
    \centering 
    \includegraphics[width=0.8\linewidth]{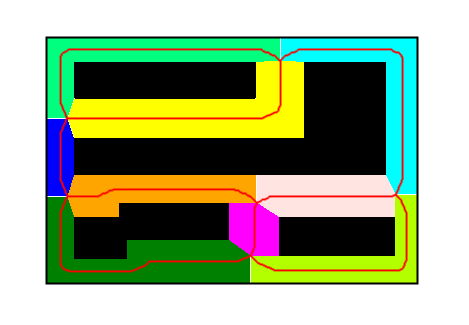} 
    \caption{The non-convex area $D$ with four black holes. The red line represents the GVG curve, and each GVG cell is denoted by a distinct color.} 
    \label{fig:2} 
\end{figure}

\begin{figure}[htbp] 
    \centering 
    \includegraphics[width=0.8\linewidth]{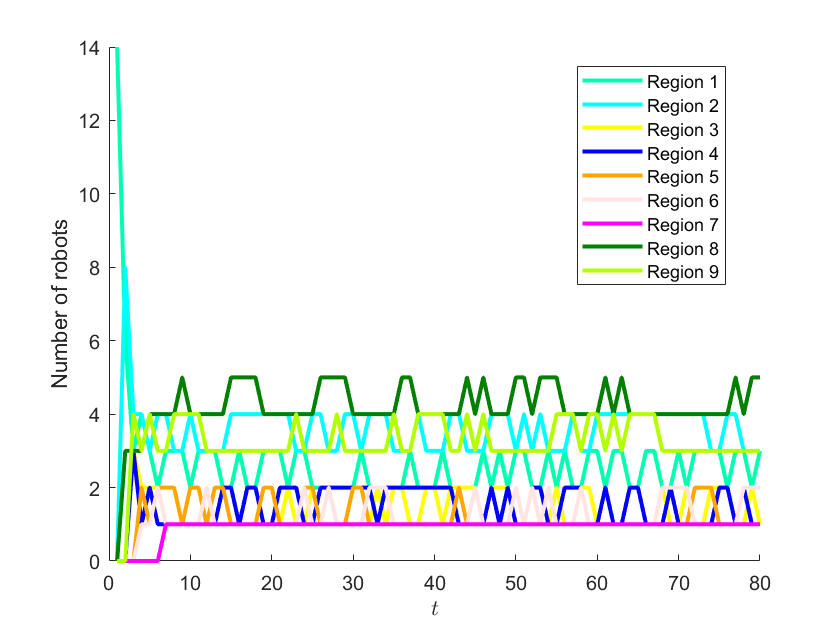} 
    \caption{The final configuration of robots in each region.} 
    \label{fig:3} 
\end{figure}

Fig.~\ref{fig:4} demonstrates the difference between the actual number and the ideal number of the robots after the execution of the algorithms. Notably, the entire system has achieved the final configuration which guarantees the load balancing. In addition, Fig.~\ref{fig:4} combined with the phenomenon shown in Fig.~\ref{fig:3}, also well verifies the conclusion of Theorem~\ref{theo:1}, namely that the robot configuration in all regions will eventually reach a state of being rounded up or rounded down relative to the ideal numbers.

\begin{figure}[htbp] 
    \centering 
    \includegraphics[width=0.9\linewidth]{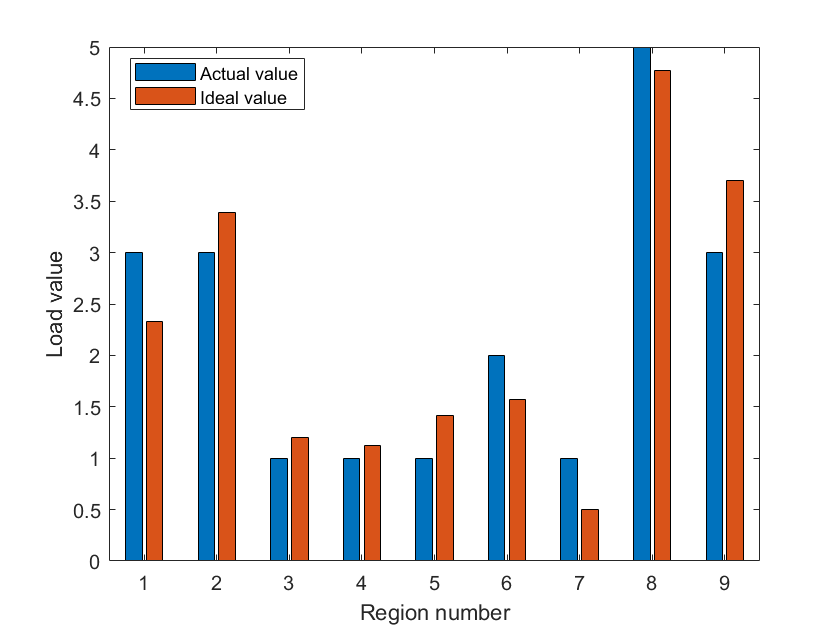} 
    \caption{A comparison between the actual number of robots and the ideal number of robots in each region.} 
    \label{fig:4} 
\end{figure}

After the execution of the Load-Balancing Algorithm, each robot is controlled to coverage the cell according to the controller~\eqref{e5}. Parameter value is set to $k_g=0.1$. As time goes on, the robots achieve complete coverage of the whole area. Fig.~\ref{fig:5c} shows the trajectories of the robots trying to coverage the cell where blue circles represent the trajectories. Fig.~\ref{fig:5d} presents the final configuration of the robots and the circles with different colors are the representations of the density. From the figures we can see that under the regulation of the controller~\eqref{e5}, each robot moves and performs coverage along the corresponding GVG curve. Furthermore, this approach effectively addresses the robot deployment issue in narrow and long regions, avoids the occurrence of suboptimal cases, and meanwhile provides a novel solution for the coverage problem in regions with non-convex boundaries.

\begin{figure}[htbp]
\centering
\subfloat[]
{\label{fig:5a}\includegraphics[width=0.5\linewidth]{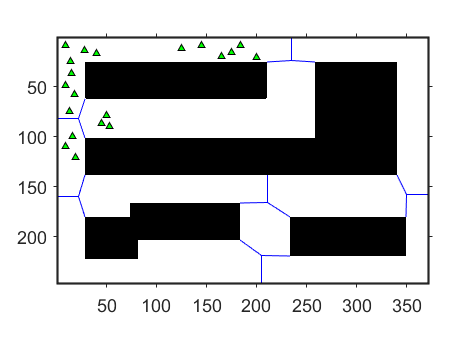}}
\subfloat[]
{\label{fig:5b}\includegraphics[width=0.5\linewidth]{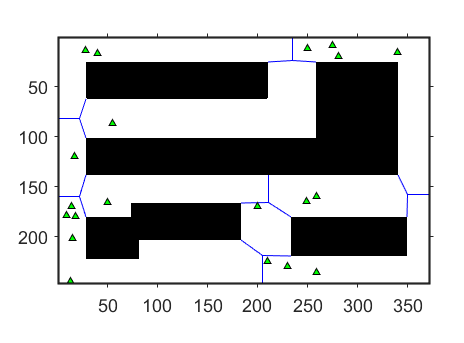}}
\\
\subfloat[]
{\label{fig:5c}\includegraphics[width=0.5\linewidth]{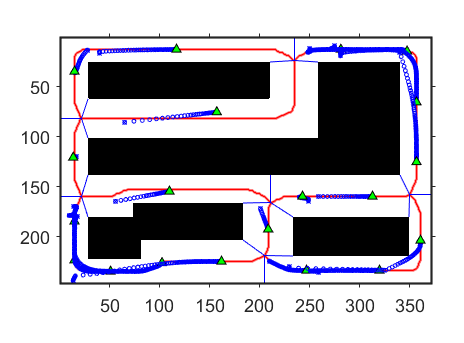}}%
\subfloat[]
{\label{fig:5d}\includegraphics[width=0.5\linewidth]{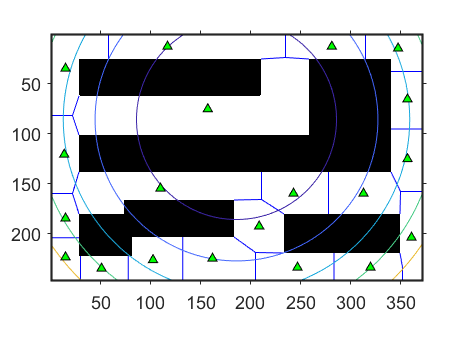}}%
\caption{The initial configuration is shown in (a) and the configuration after the execution of the Load-Balancing Algorithm is shown in (b). The trajectories of robots is shown in (c) and (d) shows the final configuration and the circles represent the density. } 
\label{fig:5} 
\end{figure}

To analyze the coverage performance of the algorithm we proposed in this paper, the following cost function is adopted
\begin{equation*}
\begin{aligned}
\mathcal{H}=10^{-3}\sum_{i = 1}^{|E|}\sum_{j = 1}^{K_i}\int_{O_j}f(p_j,q)\phi(q)\, dq . 
\end{aligned}
\end{equation*}
Fig.~\ref{fig:6} demonstrates the change of cost $\mathcal{H}$ in a given time. It is obvious that our algorithm proves to have a good coverage performance.

The simulation results show that the robots successfully coverage the whole region while avoiding the collision with the obstacles.

\begin{figure}[htbp] 
    \centering 
    \includegraphics[width=0.9\linewidth]{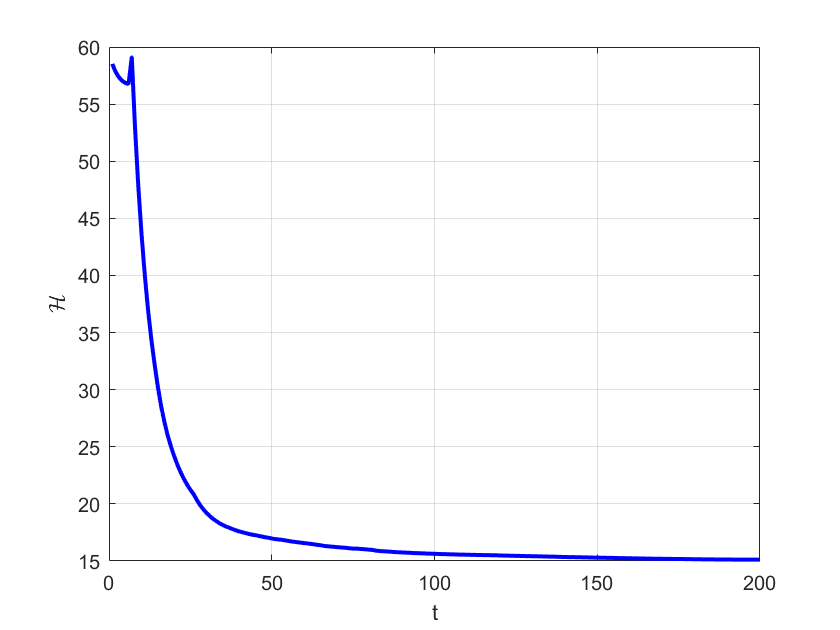} 
    \caption{Cost $\mathcal{H}$ for the area.} 
    \label{fig:6} 
\end{figure}

\section{CONCLUSIONS}
\label{sec:6}

In this paper, we have proposed a novel coverage control approach for non-convex environments with multiple obstacles based on the guidance of the generalized Voronoi graph. The convergence of the algorithms is proved and the effectiveness has been validated through simulations. In the future, our plans involve conducting real-world testing on a physical multi-robot platform. We also aim to expand the algorithms to more general dynamic non-convex environments and incorporate physical robot constraints into the collaborative coverage controller, to further bridge the gap between theoretical optimality and real-world deployment efficiency.

\addtolength{\textheight}{-12cm}   






%



\bibliographystyle{unsrt}
\bibliography{reference}

\end{document}